%% file: main.tex
\def\isarxiv{1} %%% for icml submission version, we comment this line

\ifdefined\isarxiv
\documentclass[11pt]{article}

\usepackage[numbers]{natbib}

\else
\documentclass{article}
\usepackage{neurips_2022}
\fi

\usepackage{amsmath}
\usepackage{amsthm}
\usepackage{amssymb}
\usepackage{algorithm}
\usepackage{subfig}
\usepackage{algpseudocode}
\usepackage{graphicx}
\usepackage{grffile}
\usepackage{wrapfig,epsfig}
\usepackage{url}
\usepackage{xcolor}
\usepackage{epstopdf}

\usepackage{bbm}
\usepackage{dsfont}

\allowdisplaybreaks

\ifdefined\isarxiv

\usepackage{tikz}
\usepackage{hyperref}  %%% arxiv don't allow this.
\hypersetup{colorlinks=true,citecolor=blue,linkcolor=blue} %%% Zhao : maybe we should comment this in submission.
\usetikzlibrary{arrows}
\usepackage[margin=1in]{geometry}

\else

\usepackage{microtype}
\usepackage{hyperref}
\definecolor{mydarkblue}{rgb}{0,0.08,0.45}
\hypersetup{colorlinks=true, citecolor=mydarkblue,linkcolor=mydarkblue}

\fi
 
\graphicspath{{./figs/}}

\theoremstyle{plain}
\newtheorem{theorem}{Theorem}[section]
\newtheorem{lemma}[theorem]{Lemma}
\newtheorem{definition}[theorem]{Definition}

\newtheorem{proposition}[theorem]{Proposition}
\newtheorem{corollary}[theorem]{Corollary}

\newtheorem{fact}[theorem]{Fact}
\newtheorem{remark}[theorem]{Remark}

\newcommand{\R}{\mathbb{R}}

\newcommand{\dlogtime}{\mathsf{DLOGTIME}}
\newcommand{\F}{\mathbb{F}}
\newcommand{\AC}{\mathsf{AC}}
\newcommand{\NC}{\mathsf{NC}}
\newcommand{\TC}{\mathsf{TC}}

\newcommand{\rope}{\mathsf{RoPE}}

\DeclareMathOperator{\poly}{poly}

\DeclareMathOperator{\diag}{diag}

\makeatletter
\newcommand*{\RN}[1]{\expandafter\@slowromancap\romannumeral #1@}
\makeatother

\usepackage{lineno}

\begin{document}

\ifdefined\isarxiv

\date{}

\title{Theoretical Constraints on the Expressive Power of $\mathsf{RoPE}$-based Tensor Attention Transformers}
\author{
Xiaoyu Li\thanks{\texttt{
7.xiaoyu.li@gmail.com}. Independent Researcher.}
\and
Yingyu Liang\thanks{\texttt{
yingyul@hku.hk}. The University of Hong Kong. \texttt{
yliang@cs.wisc.edu}. University of Wisconsin-Madison.} 
\and
Zhenmei Shi\thanks{\texttt{
zhmeishi@cs.wisc.edu}. University of Wisconsin-Madison.}
\and 
Zhao Song\thanks{\texttt{ magic.linuxkde@gmail.com}. The Simons Institute for the Theory of Computing at UC Berkeley.}
\and
Mingda Wan\thanks{\texttt{
dylan.r.mathison@gmail.com}. Anhui University.}
}

\else

\title{Theoretical Constraints on the Expressive Power of $\mathsf{RoPE}$-based Tensor Attention Transformers} 
\maketitle 
\fi

\ifdefined\isarxiv
\begin{titlepage}
  \maketitle
  \begin{abstract}
\input{0_abstract}

  \end{abstract}
  \thispagestyle{empty}
\end{titlepage}

{\hypersetup{linkcolor=black}
\tableofcontents
}
\newpage

\else

\begin{abstract}
\input{0_abstract}
\end{abstract}

\fi

\input{1_intro} %%% Section 1. Introduction
\input{2_related_work}
\input{3_preliminary}
\input{4_main_result}
\input{5_hardness}
\input{6_conclusion}

\ifdefined\isarxiv
%\section*{Acknowledgments}
\bibliographystyle{alpha}
\bibliography{ref}
\else
\bibliography{ref}
\bibliographystyle{alpha}

\fi

\newpage
\onecolumn
\appendix

%%%% Cut-line between first 10 pages and appendix

%%% some writing rules

%% Writing rule for creating tags.
%% Tags :
%% Theorem    \ref{thm:bla_bla}
%% Lemma      \ref{lem:bla_bla}
%% Claim      \ref{cla:bla_bla}
%% Corollary  \ref{cor:bla_bla}
%% Fact       \ref{fac:bla_bla}
%% Definition \ref{def:bla_bla}
%% Section    \ref{sec:bla_bla}
%% Subsection \ref{sub:bla_bla}
%% Equation   \ref{eq:bla_bla}

\end{document}

%% file: 0_abstract.tex
Tensor Attention extends traditional attention mechanisms by capturing high-order correlations across multiple modalities, addressing the limitations of classical matrix-based attention. Meanwhile, Rotary Position Embedding ($\mathsf{RoPE}$) has shown superior performance in encoding positional information in long-context scenarios, significantly enhancing transformer models' expressiveness. Despite these empirical successes, the theoretical limitations of these technologies remain underexplored. In this study, we analyze the circuit complexity of Tensor Attention and $\mathsf{RoPE}$-based Tensor Attention, showing that with polynomial precision, constant-depth layers, and linear or sublinear hidden dimension, they cannot solve fixed membership problems or $(A_{F,r})^*$ closure problems, under the assumption that $\mathsf{TC}^0 \neq \mathsf{NC}^1$. These findings highlight a gap between the empirical performance and theoretical constraints of Tensor Attention and $\mathsf{RoPE}$-based Tensor Attention Transformers, offering insights that could guide the development of more theoretically grounded approaches to Transformer model design and scaling.

%% file: 1_intro.tex
\section{Introduction}

Large Language Models (LLMs), such as OpenAI's ChatGPT~\cite{gpt4}, Google's Gemini~\cite{g24_gemini}, Anthropic's Claude 3.5~\cite{a24}, and Meta's LLaMA 3.3~\cite{m24} have reshaped a wide range of fields by demonstrating unprecedented advancements. These advancements are primarily due to their capability to efficiently process long-context inputs, a crucial feature for tasks like summarizing lengthy documents (e.g., medical reports, legal analyses, technical briefs), enabling superior reasoning and problem-solving performance at a level comparable to expert human analysis. 
At the core of these advancements lies the Transformer architecture~\cite{vsp+17}, driven by its self-attention mechanism. Understanding computational primitives that Transformer components enable is pivotal for principled interpretations and exposing limitations in Transformer-based systems.

Previous research has investigated these questions by analyzing the expressiveness of Transformers. As an illustration, the work in \cite{ms23} showed that constant-depth threshold circuit families can effectively emulate Transformers with precision $c \log n$ and depth-$d$. This holds true in both non-uniform and $\mathsf{L}$-uniform computational models. This result highlights Transformers' computational efficiency and structural adaptability when analyzed through circuit complexity theory's lens. Expanding on these results, \cite{chi24} showed that Transformers with $O(\log n)$ precision belong to $\dlogtime$-uniform $\TC^0$, even when the absolute error is bounded by $2^{-O(\poly(n))}$. 

To augment the capabilities of Transformers, innovations such as Rotation Position Embedding ($\rope$)\cite{sal+24} have been proposed. Through the rotation matrices, $\rope$ improves the sequence length adaptability while enhancing the efficacy of attention mechanisms. Meanwhile, multi-view approaches are increasingly recognized for capturing high-order correlations in diverse data types, including mathematical data~\cite{sht24}, graph structures~\cite{dlg+21, lyh+23}, and multi-modality datasets~\cite{laj15_multmodal}. Models like GPT-4o~\cite{o24_40} and Google's Project Astra~\cite{g24_gemini} exemplify this trend, integrating reasoning across multi-modality in real-time.
Despite these advancements, classical attention mechanisms face representational limitations. Specifically, \cite{sht24} demonstrated that matrix attention can only capture pairwise correlations, falling short in modeling triple-wise or higher-order interactions. Addressing such limitations typically requires multiple layers or carefully designed architectures, complicating the integration of multi-view information.

To overcome these constraints, \cite{sht24} and \cite{as24_iclr} proposed Tensor Attention, a higher-order extension of matrix attention. Tensor Attention intrinsically captures high-order correlations, defined as $\mathsf{Softmax}(Q (K_1 \oslash K_2)^\top)(V_1 \oslash V_2)$ (see Definition~\ref{def:single_tensor_attention}), where $\oslash$ denotes the column-wise Kronecker product (see Definition~\ref{def:tensor_oslash}). Here, $Q$, $K_1/V_1$, and $K_2/V_2$ represent inputs from different views or modalities. This raises a natural question: 
\begin{center}
    {\it Does the $\rope$ and tensor attention enhance the expressiveness of the $\rope$-based tensor attention Transformer?}
\end{center}

This work addresses this question through the lens of circuit complexity, advancing the theoretical understanding of tensor attention and $\rope$-based tensor attention mechanisms. 

We present a rigorous analysis of tensor attention Transformers and $\rope$-based tensor attention Transformers, delineating their intrinsic computational limitations. Our approach methodically evaluates the circuit complexity of each architectural component, ranging from basic trigonometric operations to the comprehensive $\rope$-based tensor attention Transformers. Specifically, it is demonstrated that uniform $\mathsf{TC}^0$ circuits are amenable to simulating the components mentioned above. Furthermore, it is proven that, unless $\mathsf{TC}^0 = \mathsf{NC}^1$, tensor attention Transformers, as well as $\rope$-enhanced tensor attention Transformers with $O(1)$ layers, $\poly(n)$-precision, and a feature dimension $d = O(n)$ are incapable of solving fixed membership problems or $(A_{F,r})^*$ closure problems. This finding underscores fundamental expressivity constraints inherent to tensor attention and $\rope$-based tensor attention architectures.

The summary of our contributions to the theoretical understanding of these architectures and their computational boundaries, rooted in circuit complexity theory, showed as follows:
\begin{itemize}
    \item Unless $\TC^0 = \NC^1$, we demonstrate that a $\dlogtime$-uniform $\TC^0$ circuit family can simulate a tensor attention Transformer or a $\rope$-based tensor attention Transformer, with constant depth, $\poly(n)$ size, and $\poly(n)$ precision (Based on Theorem~\ref{thm:main_result_tensor_attention_tc0} and Theorem~\ref{thm:main_result_tc0}).
    \item We demonstrate that, unless $\TC^0 = \NC^1$, a tensor attention Transformer or a $\rope$-based tensor attention Transformer with $O(1)$ layers, $\poly(n)$ precision, and a feature dimension $d = O(n)$ are incapable of accomplishing the fixed membership problems (Based on Theorem~\ref{thm:tc_fixed_membership} and Theorem~\ref{thm:tc_fixed_membership_tensor_attention}).
    \item We demonstrate that, unless $\TC^0 = \NC^1$, a tensor attention Transformer or a $\rope$-based tensor attention Transformer with $O(1)$ layers, $\poly(n)$ precision, and a feature dimension $d = O(n)$ are incapable of accomplishing the $(A_{F,r})^*$ closure problems (Based on Theorem~\ref{thm:tc_a_closure} and Theorem~\ref{thm:tc_a_closure_tensor_attention}).
\end{itemize}

%% file: 2_related_work.tex
\section{Related Work}\label{sec:related_work}

\paragraph{The Computational Complexity in Deep Learning.}  
Circuit complexity, a specialized domain within computational complexity theory, investigates the properties of circuit families as computational models. Numerous circuit complexity classes are relevant to the study of machine learning, with $\mathsf{AC}^0$ characterizing problems solvable by highly parallel circuits utilizing elementary logic gates. The class $\mathsf{TC}^0$ generalizes this concept by incorporating circuits that feature threshold gates, while $\mathsf{NC}^1$ encompasses problems solvable by circuits with a depth of $O(\log n)$ and bounded gate arity~\cite{mss22}. It is well-established that $\mathsf{AC}^0 \subset \mathsf{TC}^0 \subseteq \mathsf{NC}^1$, although the question of whether $\mathsf{TC}^0 \neq \mathsf{NC}^1$ remains unresolved. Assuming this inequality holds, \cite{lag+22} demonstrates that, when simulating certain non-solvable semiautomata, the depth of Transformers must necessarily increase with the input sequence length. The circuit complexity is also used to measure some other popular architectures such us Mamba \cite{cll+24_mamba} and Hopfield networks \cite{lll+24}. 

\paragraph{Computation of Transformers.}
Transformers have undeniably revolutionized the field of natural language processing, yet their performance significantly deteriorates when tasked with mathematical computations \cite{c22}. This observation has led to a surge in research aimed at identifying the computational boundaries of Transformer models, particularly in two distinct categories: (1) average-head attention Transformers, which assign a value of 1 to the highest probability in the vector while setting all other probabilities to 0, and (2) softmax-attention Transformers, which utilize the softmax function. Merrill, Sabharwal, and Smith \cite{mss22} demonstrate that average-head attention Transformers are capable of recognizing languages that surpass the computational power of the $\AC^0$ class, yet they remain simulable by threshold circuits with constant-depth, belonging to non-uniform $\TC^0$ complexity class. In a similar vein, \cite{lag+22} establish that softmax-attention Transformers also belong to the non-uniform $\TC^0$ class. Subsequent work by \cite{ms23} builds upon these findings by introducing a similarity function to demonstrate that softmax-attention Transformers fall within the $\mathsf{L}$-uniform $\TC^0$ class. Further advancements by \cite{ms24a} employ first-order logic and $\mathsf{MAJORITY}$ quantifiers \cite{i98} to show that $\dlogtime$-uniform $\TC^0$ circuits can simulate the behavior of these Transformers. In the context of practical applications, such as arithmetic operations and decision-making tasks, \cite{fzg+24} establish that unless $\mathsf{TC}^0 = \mathsf{NC}^1$, Transformers with log-precision cannot efficiently solve arithmetic problems, equation-solving tasks, or context-free grammar (CFG) membership testing \cite{s96}. These results underscore the limitations that Transformers face when applied to mathematical problems. However, recent efforts have aimed to overcome these constraints, introducing innovations such as looper Transformers \cite{as24_roap, lss+24_looped, cll+24_looped}, acceleration techniques \cite{hyw+23, lls+24_beyond, lssy24, lls+24_conv, smn+24, lls+24_fine_grained, lls+24_je, hwsl24_dits, hcl+24_outlier, hlsl24_hopfield, whl+24, xyh+24_bishop, hcw+24_nonparametric, szz21_subquadratic, whh+24_uniform, hwl24_provably,llss24, cls+24,lss+24_multi}, and other related approaches \cite{dswy22_nearly, sy23_automatic, ssx23_neighbor, gms23_regression, xsl24_llm, lls+24_grok, lssz24_dp, hsk+24_lora,kls+24, hwg+24, wsh+24}. These contributions aim to address some of the foundational limitations of Transformer architectures, allowing them to handle increasingly complex and mathematically intensive tasks.

\paragraph{Tensor Computation for High-order Representation.}
Tensors outperform matrices in capturing higher-order relationships within data. Computing low-rank factorizations or approximations of tensors is critical in various computer science applications, including natural language processing \citep{cyym14, lxz+14, lzm+15, bnr+15}, computer vision \citep{adtl09, ply10, lfc+16, clz17}, computer graphics \citep{vt04, wws+05, vas09}, security \citep{acky05, acy06, kb06}, and data mining \citep{kabo10, rs10, ks08, mor11}. Tensors also play a key role in numerous machine learning tasks \citep{pbl15, jo14, hk13, lb+13, zsj+17, ysst19, ssl+22} and other diverse domains \citep{rtp16, cmd+15, zczj14, ycs16, rnss18}.

\paragraph{Roadmap.}
In Section~\ref{sec:pre}, we introduce essential computational techniques and key definitions of Transformers that serve as the foundation for the discussions in the following sections. Section~\ref{sec:complexity_each_step} discusses the computational complexity of conventional tensor attention Transformers. Section~\ref{sec:complexity_each_step_rope} provides a detailed $\rope$-based tensor attention circuit complexity analysis. Section~\ref{sec:hardness} presents the hardness of tensor attention Transformers and $\rope$-based tensor attention Transformers derived from our study. We conclude in Section~\ref{sec:conclusion}.

%% file: 3_preliminary.tex
\section{Preliminary}\label{sec:pre}
This section establishes the essential concepts and definitions. Section~\ref{sec:pre:notation} introduces the fundamental notations that form the basis of analysis. Section~\ref{sec:pre:float} provides an in-depth exploration of float point number computation. Section~\ref{sec:pre:circuit} offers a comprehensive overview of computational complexity classes. Then, Section~\ref{sec:pre:def_tensor_oper} presents essential techniques employed in tensor operations. Finally, Section~\ref{sec:pre:trans_block} explores the fundamental components that constitute the $\rope$-based tensor attention Transformers.

\subsection{Essential Notations}\label{sec:pre:notation}  
Let $n$ represent any positive integer. The set of the first $n$ natural numbers is denoted as $[n] := \{1, 2, \dots, n\}$. The inner product of vectors $\alpha, \beta \in \mathbb{R}^n$ is given by $\langle \alpha, \beta \rangle$. The vector ${\bf 1}_n$ is an $n$-dimensional vector, where each component is one. The $\ell_\infty$ norm of a matrix $W \in \mathbb{R}^{n \times d}$ is represented as $\|W\|_\infty := \max_{m \in [n], n \in [d]} |W_{m,n}|$. Finally, a binary string $x_i \in \{0,1\}^*$ denotes a sequence of arbitrary length. 

\subsection{Float Point Operations}\label{sec:pre:float}
We present basic concepts of the computational foundation. Initially, the exact definitions of float point numbers and the corresponding operations are outlined, which are indispensable in efficient tensor attention computations.

\begin{definition}[Float point number, Definition 9 from~\cite{chi24}]
Any $p$-bit float point number is characterized by a pair $\langle r, k \rangle$, both $r$ and $k$ are integer values. Specifically, the significand of $r$ lies within the range $(-2^p, -2^{p-1}] \cup \{0\} \cup [2^{p-1}, 2^p)$, while the exponent $k$ is constrained to the interval $[-2^p, 2^p)$. The product $r \cdot 2^k$ is the real value corresponding to the float point number $\langle r, k \rangle$. The collection of all possible $p$-bit float point numbers is represented by $\mathbb{F}_p$. 
\end{definition}

\begin{definition}[Rounding, Definition 9 from~\cite{chi24}]
Given any real number or float point value $x$, the notation $\operatorname{round}_p(x)$ denotes the $p$-bit float point number closest to $x$. In cases where we have different numbers equidistant from $x$, the tie-breaking convention dictates that $\operatorname{round}_p(x)$ will be the even significand one. 
\end{definition}

Based on the foundational concepts mentioned above, we now introduce the key operations involved in tensor attention. 

\begin{definition}[Float point operations, Lemma 10 from~\cite{chi24}]\label{def:float_operations}
Let $x$ and $y$ represent two integers, then $x \oslash y$ defined as follows:
\begin{align*}
x \oslash y := 
\begin{cases}
1 / 8 + x / y & \text{if } x / y \text{ is not a multiple of } 1 / 4, \\
x / y & \text{if } x / y \text{ is a multiple of } 1 / 4.
\end{cases}
\end{align*}

Let $\langle r_1, k_1 \rangle$ and $\langle r_2, k_2 \rangle$ all denoted as $p$-bit float points, then we have: 

\begin{itemize}
  \item \textbf{Addition:}
  \ifdefined\isarxiv
  \begin{align*}
  \langle r_1, k_1 \rangle + \langle r_2, k_2 \rangle := 
  \begin{cases}
  \operatorname{round}_p(\langle r_1 + r_2 \oslash 2^{k_1 - k_2}, k_1 \rangle) & \text{if } k_1 \geq k_2, \\
  \operatorname{round}_p(\langle r_1 \oslash 2^{k_2 - k_1} + r_2, k_2 \rangle) & \text{if } k_1 \leq k_2.
  \end{cases}
  \end{align*}
  \else
  \begin{align*}
  &~ \langle r_1, k_1 \rangle + \langle r_2, k_2 \rangle \\
  := &~ 
  \begin{cases}
  \operatorname{round}_p(\langle r_1 + r_2 \oslash 2^{k_1 - k_2}, k_1 \rangle) & \text{if } k_1 \geq k_2, \\
  \operatorname{round}_p(\langle r_1 \oslash 2^{k_2 - k_1} + r_2, k_2 \rangle) & \text{if } k_1 \leq k_2.
  \end{cases}
  \end{align*}
  \fi

  \item \textbf{Comparison:}
  \begin{align*}
  \langle r_1, k_1 \rangle \leq \langle r_2, k_2 \rangle \Leftrightarrow 
  \begin{cases}
  r_1 \leq r_2 \oslash 2^{k_1 - k_2} & \text{if } k_1 \geq k_2, \\
  r_1 \oslash 2^{k_2 - k_1} \leq r_2 & \text{if } k_1 \leq k_2.
  \end{cases}
  \end{align*}

  \item \textbf{Multiplication:}
  \begin{align*}
  \langle r_1, k_1 \rangle \times \langle r_2, k_2 \rangle := \operatorname{round}_p(\langle r_1 r_2, k_1 + k_2 \rangle).
  \end{align*}

  \item \textbf{Division:}
  \ifdefined\isarxiv
  \begin{align*}
  \langle r_1, k_1 \rangle \div \langle r_2, k_2 \rangle := \operatorname{round}_p(\langle r_1 2^{p-1} \oslash r_2, k_1 - k_2 - p + 1 \rangle).
  \end{align*}
  \else
  \begin{align*}
  &~ \langle r_1, k_1 \rangle \div \langle r_2, k_2 \rangle \\
  := &~ \operatorname{round}_p(\langle r_1 2^{p-1} \oslash r_2, k_1 - k_2 - p + 1 \rangle).
  \end{align*}
  \fi

  \item \textbf{Floor:}
  \begin{align*}
  \lfloor \langle r, k \rangle \rfloor := 
  \begin{cases}
  \operatorname{round}(\langle r / 2^{-k}, 0 \rangle) & \text{if } k < 0, \\
  \langle r 2^k, 0 \rangle & \text{if } k \geq 0.
  \end{cases}
  \end{align*}

\end{itemize}
\end{definition}
The operations mentioned above are capable of efficient hardware implementation, as demonstrated by the following lemmas:

\begin{lemma}[Float point operations in $\TC^0$, Lemma 10 and Lemma 11 from~\cite{chi24}]\label{lem:float_operations_TC} 
If integer $0 < p \leq \poly(n)$, 
then we say the conditions below are satisfied:
\begin{itemize}

\item \textbf{Part 1.} As we described in Definition~\ref{def:float_operations}, the operations addition, division, multiplication, and comparison of two p-bit float point numbers are calculable by a constant depth $\poly(n)$ size uniform threshold circuit. $d_\mathrm{std}$ denotes the deepest depth necessitated for executing these operations.

\item \textbf{Part 2.} We can execute $n$ $p$-bit float point numbers repeated multiplication using a constant depth $\poly(n)$ size uniform threshold circuit. The required depth for this iterated multiplication process is denoted as $d_\otimes$.

\item \textbf{Part 3.} We can approximate $n$ $p$-bit float point numbers sequential addition and rounding using a constant depth $\poly(n)$ size uniform threshold circuit. The depth needed for iterated addition is represented by $d_\oplus$.

\end{itemize}
\end{lemma}

\begin{corollary}[Floor operation in $\TC^0$, Corollary 3.17 from~\cite{cll+24_rope}]\label{cor:floor_float_TC}
For any integer $0 < p \leq \poly(n)$, a $\poly(n)$ size constant depth uniform threshold circuit is able to calculate the floor operation from Definition~\ref{def:float_operations} on a $p$-bit float point number. The operation's maximum depth is bounded by $d_\mathrm{std}$, as established in Lemma~\ref{lem:float_operations_TC}.
\end{corollary}

\begin{lemma}[Computing $\exp$ in $\TC^0$, Lemma 12 from~\cite{chi24}]\label{lem:exp}
For any integer $0 < p \leq \poly(n)$ and any $p$-bit float point number $x$, it is computable to approximate most $2^{-p}$ relative error $\exp(x)$ using $\poly(n)$ size constant depth uniform threshold circuit. The depth required for this computation is denoted by $d_{\exp}$.
\end{lemma}

\begin{lemma}[Computing square root in $\TC^0$, Lemma 12 from~\cite{chi24}]\label{lem:sqrt}
Given an integer $p$ such that $0 < p \leq \poly(n)$ and a $p$-bit float point number $x$, a constant depth $\poly(n)$ size uniform threshold circuit exists to calculate $\sqrt{x}$ with a relative error bounded by $2^{-p}$. The depth required for this operation is represented by $d_\mathrm{sqrt}$.
\end{lemma}

\subsection{Circuit Complexity}\label{sec:pre:circuit}  
In computational theory, a Boolean circuit, constructed using basic gates such as $\mathsf{AND}$, $\mathsf{OR}$, and $\mathsf{NOT}$, represents a core model of computation. A precise mathematical definition of this structure comes below. 

\begin{definition}[Boolean Circuit, Definition 6.1 from~\cite{ab09}]
An $n$ variables Boolean circuit is defined as $C_n: \{0, 1\}^n \to \{0, 1\}$ and is depicted by a directed acyclic graph (DAG). In this representation, logical gates such as $\mathsf{AND}$, $\mathsf{OR}$, and $\mathsf{NOT}$ correspond to the vertices of the graph. The input vertices, each linked to one of the $n$ Boolean variables, have an in-degree of 0, whereas non-input vertices derive their values from the outputs of preceding gates in the structure.
\end{definition}

\begin{definition}[Languages, Definition 6.2 from~\cite{ab09}]
A Boolean circuit family $\mathcal{C}$ is said to recognize language $L \subseteq \{0, 1\}^*$ if a Boolean circuit $C_{| z |} \in \mathcal{C}$ with $| z |$ variables exists, s.t., $C_{| z |}(z) = 1$, iff $z \in L$, for every string $z \in \{0, 1\}^*$. 
\end{definition}

\begin{definition}[$\NC^i$, Definition 6.21 from~\cite{ab09}]
The class $\NC^i$ is defined as the set of languages that are recognizable using Boolean circuits of size $O(\poly(n))$ and depth $O((\log n)^i)$, with logical gates of bounded fan-in, including $\mathsf{NOT}$, $\mathsf{OR}$, and $\mathsf{AND}$ gates.
\end{definition}

When Boolean circuits are permitted to incorporate gates such as $\mathsf{AND}$ and $\mathsf{OR}$ with unbounded fan-in, their ability to process languages becomes significantly enhanced. This development leads to the introduction of the complexity class $\mathsf{AC^i}$.

\begin{definition}[$\AC^i$, Definition 6.22 from~\cite{ab09}]
Languages which can be computed by the Boolean circuit of depth $O((\log n)^i)$, size $O(\poly(n))$, unbounded fan-in gates, including $\mathsf{AND}$, $\mathsf{OR}$, $\mathsf{NOT}$, are contained in the class $\AC^i$.
\end{definition}

The $\mathsf{MAJORITY}$ gates can simulate $\mathsf{AND}$, $\mathsf{NOT}$, $\mathsf{OR}$ gates, which yield an output of 1 if the majority of inputs are 1, and 0 otherwise. By incorporating $\mathsf{MAJORITY}$ gates, one can define a broader complexity class known as $\TC^i$.

\begin{definition}[$\TC^i$, Definition 4.34 from~\cite{vol99}]\label{def:tc}
If we have languages are recognizable by $O(\poly(n))$ size Boolean circuits of $O((\log n)^i)$ depth, and unbounded fan-in gates, including $\mathsf{MAJORITY}$, $\mathsf{NOT}$, $\mathsf{OR}$, and $\mathsf{AND}$ gates. If half of the inputs are 1, the $\mathsf{MAJORITY}$ gate will output 1. 

The class $\TC^i$ contains languages that are recognizable by Boolean circuits of size $O(\poly(n))$, depth $O((\log n)^i)$, and gates with unbounded fan-in, including $\mathsf{NOT}$, $\mathsf{OR}$, $\mathsf{AND}$, and $\mathsf{MAJORITY}$ gates. A $\mathsf{MAJORITY}$ gate outputs one if more than half of its inputs are one.
\end{definition}

\begin{remark}
As Definition~\ref{def:tc} shows, $\mathsf{MOD}$ or $\mathsf{THRESHOLD}$ gates (for prime moduli) can replace $\mathsf{MAJORITY}$ gates. Boolean circuits employing such gates are collectively referred to as threshold circuits.
\end{remark}

Next, we formally introduce the class $\mathsf{P}$.

\begin{definition}[$\mathsf{P}$, Definition 1.20 from~\cite{ab09}]
A language is considered to be in $\mathsf{P}$ if it can be decided by a deterministic Turing machine within polynomial time of input size. 
\end{definition}

The hierarchical relationships among certain circuit families are encapsulated in the following well-known result.

\begin{fact}[Corollary, Corollary 4.35 from~\cite{vol99}]\label{fact:folklore}
Any $i \in \mathbb{N}$, the following inclusions are valid:
\[
\NC^i \subseteq \AC^i \subseteq \TC^i \subseteq \NC^{i+1} \subseteq \mathsf{P}.
\]
\end{fact}

\begin{remark}
If $i = 0$, it has been established $\NC^0 \subsetneq \AC^0 \subsetneq \TC^0$. However, it remains unresolved whether $\TC^0 \subsetneq \NC^1$. Moreover, the question of whether $\NC := \cup_{i \in \mathbb{N}} \NC^i \subsetneq \mathsf{P}$ is an open problem. Additional details can be found in Corollary 4.35 from~\cite{vol99}.
\end{remark}

Non-uniform circuit families, characterized by their lack of consistent structural design across varying input sizes, are theoretically capable of addressing undecidable problems. Nevertheless, their impracticality arises from the infinite length required for their description. In contrast, uniform circuit families, which adhere to a systematic computational model, hold greater relevance in the study of complexity and formal language theory. We begin with the definition of $\mathsf{L}$-uniformity.

\begin{definition}[$\mathsf{L}$-uniformity class, Definition 6.5 from~\cite{ab09}]
Denote $\mathsf{C}$ as a class of languages represented by circuit family $\mathcal{C}$ (such as, $\NC^i$, $\AC^i$, or $\TC^i$). A language $L \subseteq \{0, 1\}^*$ is classified as belonging to the $\mathsf{L}$-uniform class of $\mathsf{C}$ if existing a Turing machine can map $1^n$ to $\mathcal{C}$ class circuit with $n$ variables in $O(\log n)$ space, for each $n \in \mathbb{N}$, and the resulting circuit $C_n$ recognizes $L$.
\end{definition}

Then, the $\dlogtime$-uniformity and examine its correspond to $\mathsf{L}$-uniformity will be introduced. 

\begin{definition}[$\dlogtime$-uniformity, Definition 4.28 from~\cite{bi94}]
Let $\mathsf{C}$ be a class of languages represented by circuit family $\mathcal{C}$ (such as $\NC^i$, $\AC^i$, or $\TC^i$). A language $L \subseteq \{0, 1\}^*$ is defined to belong to the $\dlogtime$-uniform class of $\mathcal{C}$ if a random-access Turing machine can map $1^n$ to $n$ variables circuit $C_n$ in $\mathcal{C}$ within $O(\log n)$ time, for every $n \in \mathbb{N}$, such that $C_n$ recognizes $L$.
\end{definition}

\begin{remark}
The concept of $\mathsf{DLOGTIME}$-uniformity aligns with that of $\mathsf{L}$-uniformity, except in smaller circuit classes that do not have the capability to imitate the constructing machine. Further exploration of uniformity concepts can be found in~\cite{bi94, hab02}. Within this paper, references to uniform $\TC^0$ pertain specifically to $\mathsf{DLOGTIME}$-uniform $\TC^0$.
\end{remark}

\subsection{Tensor Operation Analysis Techniques}\label{sec:pre:def_tensor_oper}
We first define operations such as the Kronecker product, a matrix operation that takes two matrices of any size and produces a block matrix. Unlike standard matrix multiplication, it is useful for introducing and analyzing tensor attention. Then, we introduce some key techniques for applying tensor attention to RoPE. 

\begin{definition}[$\otimes$ Kronecker product]\label{def:tensor_otimes}
Given $K_1 \in \mathbb{R}^{n_1 \times d_1}$ and $K_2 \in \mathbb{R}^{n_2 \times d_2}$, let $K := K_1 \otimes K_2 \in \mathbb{R}^{n_1 n_2 \times d_1 d_2}$ be defined for any $i_1 \in [n_1], j_1 \in [d_1] $ and $ i_2 \in [n_2],  j_2 \in [d_2]$ as
\begin{align*}
K_{i_1 + (i_2-1)n_1 , j_1 + (j_2-1)d_1} = (K_1)_{i_1,j_1} \cdot (K_2)_{i_2,j_2}.
\end{align*}
\end{definition}

\begin{definition}[$\oslash$ column-wise Kronecker product]\label{def:tensor_oslash}
Given matrices $K_1 \in \R^{n_1 \times d}, K_2 \in \R^{n_2 \times d}$, we define matrix $K := K_1 \oslash K_2 \in \R^{n_1 n_2 \times d}$ as follows
\begin{align*}
    K_{i_1 + (i_2-1)n_1 , j } := (K_1)_{i_1,j} \cdot (K_2)_{i_2,j}, ~~~\forall i_1 \in [n_1], i_2 \in [n_2], j \in [d].
\end{align*}
\end{definition}

\begin{definition}[$\ominus$ row-wise Kronecker product]\label{def:tensor_ominus}
Given matrices $K_1 \in \R^{n \times d_1}, K_2 \in \R^{n \times d_2}$, we define matrix $K := K_1 \ominus K_2 \in \R^{n \times d_1 d_2}$   
as follows
\begin{align*}
    K_{i , j_1 + (j_2-1) d_1 } := (K_1)_{i,j_1} \cdot (K_2)_{i,j_2}, ~~~\forall i \in [n], j_1 \in [d_1], j_2 \in [d_2].
\end{align*}
\end{definition}

Fact~\ref{fact:cdot_oslash_swap_informal} indicates that the order of tensor operation and matrix multiplication can be swapped, enabling computation in the lower dimension first to reduce complexity.

\begin{fact}[Swap rule for tensor and matrix product]\label{fact:cdot_oslash_swap_informal}
    Let $W_1, W_2 \in \R^{d \times d}, A_1, A_2 \in \R^{n\times d}$. We have
    \begin{align*}
        \underbrace{(A_1 \otimes A_2)}_{{n^2\times d^2}} \cdot \underbrace{(W_1 \oslash W_2)}_{{d^2\times d}} = \underbrace{(A_1 \cdot W_1)}_{{n \times d}} \oslash \underbrace{(A_2 \cdot W_2)}_{{n \times d}} .
    \end{align*}
\end{fact}
\begin{proof}
For any $i_1, i_2 \in [n], j \in [d]$, we have 
\begin{align*}
    & ~ ((A_1 \otimes A_2) \cdot (W_1 \oslash W_2))_{i_1+(i_2-1)n, j} \\ 
    = & ~ \sum_{k_1 \in [d], k_2 \in [d]}  (A_1 \otimes A_2)_{i_1+(i_2-1)n, k_1+(k_2-1)d} (W_1 \oslash W_2)_{k_1+(k_2-1)d, j}\\
    = & ~ \sum_{k_1 \in [d], k_2 \in [d]}  (A_1 \otimes A_2)_{i_1+(i_2-1)n, k_1+(k_2-1)d} \cdot (W_1)_{k_1, j} \cdot (W_2)_{k_2, j}\\
    = & ~ \sum_{k_1 \in [d], k_2 \in [d]}  (A_1)_{i_1,k_1} \cdot (A_2)_{i_2,k_2} \cdot (W_1)_{k_1, j} \cdot (W_2)_{k_2, j}\\
    = & ~ ( \sum_{k_1 \in [d]}  (A_1)_{i_1,k_1}   \cdot (W_1)_{k_1, j} ) \cdot ( \sum_{k_2 \in [d]} (A_2)_{i_2,k_2} \cdot (W_2)_{k_2, j} ) \\
    = & ~  (A_1 \cdot W_1)_{i_1, j}  \cdot (A_2 \cdot W_2)_{i_2, j} \\
    = & ~ ((A_1 \cdot W_1) \oslash (A_2 \cdot W_2))_{i_1+(i_2-1)n, j},
\end{align*}
where the initial step involves the application of matrix multiplication, followed by the utilization of Definition~\ref{def:tensor_oslash} in the second step. Subsequently, the third step employs Definition~\ref{def:tensor_otimes}, while the fourth step simplifies the expression through fundamental algebraic principles. The fifth step re-engages matrix multiplication, and the concluding step leverages Definition~\ref{def:tensor_oslash} once more.
\end{proof}

\subsection{Transformer Block}\label{sec:pre:trans_block}
With the mathematical foundation in place, this section outlines the key components of the $\rope$-based tensor attention Transformers architecture, starting with the softmax operation, a fundamental element of Transformer.

\begin{definition}[Softmax function]
Noted $z \in \F_p^{n}$. The
$\mathsf{Softmax}$ function $: \F_p^{n} \to \F_p^{n}$ is formally given by:
\begin{align*}
\mathsf{Softmax}(z):= \frac{\exp(z)}{\langle \exp(z), {\bf 1}_n \rangle}.
\end{align*}
\end{definition}

One of the pivotal advancements in contemporary Transformer architectures is $\rope$, which employs a rotation matrix as its foundation:

\begin{definition}[Rotation matrix block]\label{def:rotated_matrix}
For an input sequence of length $n$, embedding dimension $d$, and parameter $\theta \in \F_p$, the rotation matrix is constructed as follows:
\begin{align*}
    R(\theta) := \begin{bmatrix}
        \cos \theta &  -\sin \theta\\
        \sin \theta & \cos \theta
    \end{bmatrix}.
\end{align*}
\end{definition}

This fundamental rotation matrix is generalized to encode the relative positions within a sequence, facilitating the embedding of positional context.

\begin{definition}[Rotation matrix]\label{def:token_position_Rotated_matrix}
Noted $j$ represents position index within input sequence and $i$ denotes token index. The relative rotation matrix is then expressed as:
\begin{align*}
    R_{j-i}
    = \begin{bmatrix}
        R((j-i) \theta_1) & 0 & \cdots & 0 \\
        0 & R ((j-i) \theta_2) & \cdots & 0 \\
        \vdots & \vdots & \ddots & \vdots \\
       0 & 0 & \cdots & R ((j-i)\theta_{d/2}) \\
    \end{bmatrix},
\end{align*}
where the angular frequencies $\theta_1, \cdots, \theta_{d/2}$ are all predefined. More about selecting \( \theta \), consult Equation (15) from \cite{sal+24}.
\end{definition}

Leveraging rotation matrices mentioned above, $\rope$-based tensor attention embeds positional relation intrinsically within the computational process of attention. Now, we are about to introduce the $\rope$-based tensor attention. First, we introduce the parameters and input. 

\begin{definition}[Input and weight matrix]\label{def:input}
    We define the input sequence as $X \in \R^{n\times d}$ and the key, query, and value weight matrix as $W_{K_1}, W_{K_2}, W_Q, W_{V_1}, W_{V_2} \in \R^{d\times d}$. Then, we define the key, query, and value matrix as $K_1:=X W_{K_1} \in  \R^{n \times d}$, $K_2:=X W_{K_2} \in  \R^{n \times d}$, $Q:=X W_Q \in  \R^{n \times d}$, $V_1:=X W_{V_1} \in  \R^{n \times d}, V_2:=X W_{V_2} \in  \R^{n \times d}$. 
\end{definition}

Then, based on Definition~\ref{def:tensor_oslash}, we define $\rope$-based tensor attention matrix in the following way.

\begin{definition}[$\rope$-based tensor attention]\label{def:rope_attn_matrix}
As we defined in  Definition~\ref{def:token_position_Rotated_matrix} and~\ref{def:input}. We compute the new attention matrix $A \in \F_{p}^{n \times n^2}$ by, 
\begin{align*}
A_{j_1,j_2 + (j_3 - 1)d} := ( \exp(\underbrace{Q_{j_1,*}}_{1 \times d} \underbrace{R_{j_1,j_2 + (j_3 - 1)d}}_{d \times d^2} \underbrace{(K_{*,j_2 + (j_3 - 1)d})^\top}_{d^2 \times 1} /d) )_{j_1,j_2 + (j_3 - 1)d}
\end{align*}
where $R_{j_1,j_2 + (j_3 - 1)d} = \underbrace{R_{j_1-j_2}}_{d \times d} \ominus \underbrace{R_{j_1-j_3}}_{d \times d} \in \F_p^{n \times n}$, $K = \underbrace{K_1}_{n \times d} \otimes \underbrace{K_2}_{n \times d} \in \F_p^{n^2 \times d^2}$. 

\end{definition}

\begin{definition}[Single $\rope$-based tensor attention layer, Definition 7 in~\cite{sht24}, Definition 1.1 in~\cite{as24_iclr}, Definition 3.8 in~\cite{lssz24_tensor}]\label{def:single_rope_tensor_attention}
Given input matrices $Q, K_1, K_2, $ $V_1, $ $V_2 \in \F_p^{n \times d}$, $R \in \F_p^{d \times d}$, as Definition~\ref{def:rope_attn_matrix}, we compute the $i$-th $\rope$-based tensor attention layer $\mathsf{Attn}_i$ as 
\begin{align*}
\mathsf{Attn}_i (X) := \underbrace{D^{-1}}_{n \times n} \underbrace{A}_{n \times n^2} (\underbrace{X}_{n \times d} \otimes \underbrace{X}_{n \times d})(\underbrace{W_{V_{1}}}_{d \times d} \oslash \underbrace{W_{V_{2}}}_{d \times d})
\end{align*}
by applying Fact~\ref{fact:cdot_oslash_swap_informal}, we finally define the $i$-th $\rope$ tensor attention layer $\mathsf{Attn}_i$ as 
\begin{align*}
\mathsf{Attn}_i (X) := \underbrace{D^{-1}}_{n \times n} \underbrace{A}_{n \times n^2} \underbrace{V}_{n^2 \times d}
\end{align*}
where $D := \diag(\underbrace{A}_{n \times n^2} \underbrace{{\bf 1}_{n^2}}_{n^2 \times 1}) \in \F_p^{n \times n}$, and $V = \underbrace{V_1}_{n \times d} \oslash \underbrace{V_2}_{n \times d} \in \F_p^{n^2 \times d}$

\end{definition}

Then, we introduce a single tensor attention layer. 

\begin{definition}[Single tensor attention layer, Definition 7 in~\cite{sht24}, Definition 1.1 in~\cite{as24_iclr}, Definition 3.5 in~\cite{lssz24_tensor}]\label{def:single_tensor_attention}
Given input matrices $Q, K_1, K_2, $ $V_1, $ $V_2 \in \F_p^{n \times d}$, compute the following matrix 
\begin{align*}
\mathsf{Attn}_i (X) := \underbrace{D^{-1}}_{n \times n} \underbrace{A}_{n \times n^2} \underbrace{V}_{n^2 \times d}.
\end{align*}
where (1) $A := \exp(\underbrace{Q}_{n \times d} \underbrace{K^\top}_{d \times n^2} /d) \in \F_p^{n \times n^2}$ and $K := \underbrace{K_1}_{n \times d} \oslash \underbrace{K_2}_{n \times d} \in \F_p^{n^2 \times d}$, (2)  $D := \diag(\underbrace{A}_{n \times n^2} \underbrace{{\bf 1}_{n^2})}_{n^2 \times 1} \in \F_p^{n \times n}$, and (3) $V := \underbrace{V_1}_{n \times d} \oslash \underbrace{V_2}_{n \times d} \in \F_p^{n^2 \times d}$.
\end{definition}

Next, we can also integrate multi-layer attention and the additional mechanism mentioned above to construct a comprehensive Transformer. 

\begin{definition}[Multiple layer tensor attention Transformer] \label{def:multi_layer_self_attn}
The number of Transformer's layers is denoted by $m$. In the $i$-th Transformer layer, let $g_i$ signify components distinct from self-attention, where $g_i: \F_p^{n \times d} \to \F_p^{n \times d}$, each $i \in [m]$.
And $\mathsf{Attn}_i$ represent i-th layer attention mechanism(as defined in Definition~\ref{def:single_rope_tensor_attention} and Definition~\ref{def:single_tensor_attention}).
Given an input data matrix $X \in \F_p^{n \times d}$, an $m$-layer Transformer $\mathsf{TF}: \F_p^{n \times d} \to \F_p^{n \times d}$ is formally defined as:
\begin{align*}
\mathsf{TF}(X) := g_m \circ \mathsf{Attn}_m \circ \dots \circ g_1 \circ \mathsf{Attn}_1 \circ g_0 (X) ~~ \in \F{p}^{n \times d},
\end{align*}
where $\circ$ denotes the composition of functions.
\end{definition}

Subsequently, we define two categories of $g_i$ functions. Start with the layer normalization.

\begin{definition}[Layer normalization]\label{def:layer_norm}
Let $X \in \F_p^{n \times d}$ be the input data matrix, and let $i \in [n]$. The LN layer is formulated as:
\begin{align*}
g^{\mathrm{LN}} (X)_{i,*}  :=  \frac{X_{i,*} - \mu_i}{\sqrt{\sigma_i^2}},
\end{align*}
where $\mu_i := \sum_{j=1}^d \frac{X_{i,j}}{d} $, and $\sigma_i^2 := \sum_{j = 1}^d \frac{(X_{i,j} - \mu_i)^2 }{d}$.
\end{definition}

The second category is the multilayer perceptron.

\begin{definition}[Multilayer perceptron]\label{def:mlp}
Let $X \in \F_p^{n \times d}$ be the input data matrix, and let $i \in [n]$. The MLP layer is described as:
\begin{align*}
    g^{\mathrm{MLP}}(X)_{i,*} := \underbrace{W}_{d \times d} \cdot \underbrace{X_{i,*}}_{d \times 1} + \underbrace{b}_{d \times 1}.
\end{align*}
\end{definition}

The foundation of modern Transformer is built upon these layered architectures, which integrate float point computations, attention, and rotation matrix to an exceptionally efficient framework for sequential computation. 

%% file: 4_main_result.tex
\section{Complexity of Tensor Attention Transformer}\label{sec:complexity_each_step}

We now formally turn our attention to investigating the circuit complexity of the tensor attention layer and the multi-layer tensor attention Transformer, emphasizing their computability within the complexity class $\TC^0$. Section~\ref{sec:compute_matrix_product} delves into matrix operations. Section~\ref{sec:compute_tensor_attention_layer} addresses the computation of a single tensor attention layer. Section~\ref{sec:compute_multi_tensor_attention_layer} provides an in-depth examination of the entire tensor attention mechanism. Lastly, Section~\ref{sec:circuit_complexity_bound_tensor_attention} presents our principal findings regarding the circuit complexity bounds for the tensor attention Transformer. These results establish the foundation for the main theorem concerning Transformer expressiveness.

\subsection{Matrix Operations}\label{sec:compute_matrix_product}
We demonstrate that fundamental matrix multiplication is efficiently evaluatable within $\TC^0$.

\begin{lemma}[Matrix multiplication in $\TC^0$, Lemma~4.2 from~\cite{cll+24_rope}]\label{lem:matrix_multi}  
Let $A \in \F_p^{n_1 \times d}$ and $B \in \F_p^{d \times n_2}$ represent matrices. Under the conditions that $p \leq \poly(n)$, $n_1, n_2 \leq \poly(n)$, and $d \leq n$, the product $AB$ is evaluatable via $\poly(n)$ size uniform threshold circuit with $(d_\mathrm{std} + d_\oplus)$ depth. 
\end{lemma} 

\begin{lemma}[Kronecker product in $\TC^0$]\label{lem:kronecker_product}
Let $A \in \F_p^{n_1 \times d}$ and $B \in \F_p^{d \times n_2}$ represent matrices.  If $p \leq \poly(n)$, $n_1, n_2 \leq \poly(n)$, and $d \leq n$, the Kronecker product $A \otimes B$ can be evaluated by a $\poly(n)$ size uniform threshold circuit with $d_\mathrm{std}$ depth.  
\end{lemma}

\begin{proof}
Each product $(A)_{i_1, j_1} \cdot (B)_{i_2, j_2}$ computes the entry $(A \otimes B)_{i_1 + (i_2 - 1)n_1 , j_1 + (j_2 - 1)d}$, according to Part 1 of Lemma~\ref{lem:float_operations_TC}. Since the computations for distinct index pairs $(i_1, j_1)$ and $(i_2, j_2)$ are independent, they can be performed concurrently, resulting in a total depth of $d_\mathrm{std}$ for all computations.  

The circuit size is polynomial in $n$, as each operation uses a polynomial-sized circuit, and $n_1, n_2, d \leq \poly(n)$.  

Therefore, the Kronecker product $A \otimes B$ is evaluatable by a $\poly(n)$ size uniform threshold circuit with $d_\mathrm{std}$ depth.

This concludes the proof.  
\end{proof}

\begin{lemma}[Column-wise Kronecker Product in $\TC^0$]\label{lem:column_kronecker_product}
Let matrices $A \in \F_p^{n_1 \times d}$ and $B \in \F_p^{n_2 \times d}$ be given. If $p \leq \poly(n)$, $n_1, n_2 \leq \poly(n)$, and $d \leq n$, then the column-wise Kronecker product $A \oslash B$ is evaluatable by a $\poly(n)$ size uniform threshold circuit with depth $d_\mathrm{std}$.  
\end{lemma}

\begin{proof}
This result directly follows from Lemma~\ref{lem:kronecker_product}. By applying Lemma~\ref{lem:float_operations_TC}, the product $(A)_{i_1, j} \cdot (B)_{i_2, j}$ for $i_1 \in [n_1]$, $i_2 \in [n_2]$, and $j \in [d]$ computes the entry $(A \oslash B)_{i_1 + (i_2-1)n_1 , j }$ using a uniform threshold circuit with depth $d_\mathrm{std}$. Since these computations are independent for distinct values of $(i_1, i_2)$, they can be evaluated concurrently, resulting in a circuit depth of $d_\mathrm{std}$.

The circuit size remains polynomial in $n$ because $n_1, n_2, d \leq \poly(n)$ and every operation utilizes a polynomial-sized circuit.

Thus, the column-wise Kronecker product $A \oslash B$ can be evaluated by a $\poly(n)$ size depth $d_\mathrm{std}$ uniform threshold circuit.

This concludes the proof.  
\end{proof}

\begin{lemma}[Row-wise Kronecker Product Computation in $\TC^0$]\label{lem:row_kronecker_product}
Let $A \in \F_p^{d \times n_1}$ and $B \in \F_p^{d \times n_2}$ be matrices, with the conditions $p \leq \poly(n)$, $n_1, n_2 \leq \poly(n)$, and $d \leq n$. Then, a size $\poly(n)$ uniform threshold circuit with $d_\mathrm{std}$ depth can calculate the row-wise Kronecker product $A \ominus B$.
\end{lemma}

\begin{proof}
Similarly as Lemma~\ref{lem:column_kronecker_product}, according to Lemma~\ref{lem:float_operations_TC}, the product $(A)_{i, j_1} \cdot (B)_{i, j_2}$, for $j_1 \in [n_1]$, $j_2 \in [n_2]$, and $i \in [d]$, computes the entry $(A \ominus B)_{i, j_1 + (j_2-1)n_1}$ via a depth $d_\mathrm{std}$ uniform threshold circuit. These products, for distinct $(i_1, i_2)$, are evaluatable in parallel, allowing all necessary products $(A)_{i, j_1} \cdot (B)_{i, j_2}$ to be evaluated simultaneously within the depth $d_\mathrm{std}$.

The circuit size is polynomial in $n$ because $n_1, n_2, d \leq \poly(n)$, and each individual operation can be evaluated by a polynomial-sized circuit.

Hence, $\poly(n)$ size $d_\mathrm{std}$ depth uniform threshold circuit can calculate $A \ominus B$.

The proof is concluded. 
\end{proof}

\subsection{Single Tensor Attention Layer}\label{sec:compute_tensor_attention_layer}

Here, we examine the complexity of the single layer of the tensor attention.

\begin{lemma}[Complexity of Single Tensor Attention Layer in $\TC^0$]\label{lem:single_tensor_attention}
When $p \leq \poly(n)$, the attention  $\mathsf{Attn}$ in Definition~\ref{def:single_tensor_attention}, is evaluatable by a $\poly(n)$ size and $5 d_{\mathrm{std}} + 5 d_{\oplus} + d_{\mathrm{exp}}$ depth uniform threshold circuit.
\end{lemma}

\begin{proof}
The matrix multiplications $Q := Z W_Q$, $K_1 := Z W_{K_1}$, and $K_2 := Z W_{K_2}$ can be evaluated in parallel with a size $\poly(n)$ depth $d_{\mathrm{std}} + d_{\oplus}$ uniform threshold circuit, as established in Lemma~\ref{lem:matrix_multi}.

As per Lemma~\ref{lem:column_kronecker_product}, the column-wise Kronecker product $V := V_1 \oslash V_2$ is evaluatable for $\poly(n)$ size uniform threshold circuit with $d_{\mathrm{std}}$ depth.

Using Lemma~\ref{lem:matrix_multi} and Part 1 of Lemma~\ref{cor:floor_float_TC}, the operation $Q K^\top / d$ is evaluatable by $\poly(n)$ size uniform threshold circuit with depth $2 d_{\mathrm{std}} + d_{\oplus}$.

According to Lemma~\ref{lem:exp}, the exponential function $exp()$ is evaluatable by $\poly(n)$ size uniform threshold circuit with depth $d_{\mathrm{exp}}$.

As per Part 3 of Lemma~\ref{lem:float_operations_TC}, $D := A {\bf{1}}n$ is evaluated with $\poly(n)$ size uniform threshold circuit of depth $d{\oplus}$.

Finally, the expression $D^{-1} A V$ is evaluated in parallel using $\poly(n)$ size uniform threshold circuit of $2(d_{\mathrm{std}} + d_{\oplus})$ depth, as shown in Lemma~\ref{lem:matrix_multi}.

The total depth required for computing $\mathsf{Attn}_i(X) := D^{-1} A V$ is therefore: 

\begin{align*}
6 d_{\mathrm{std}} + 5 d_{\oplus} + d_{\mathrm{exp}}.
\end{align*}
\end{proof}

\subsection{Multi-layer Tensor Attention}\label{sec:compute_multi_tensor_attention_layer}

This section analyzes the computation of multi-layer tensor attention in a Transformer.

\begin{lemma}[Computation of Multi-layer Tensor Attention Transformer in $\TC^0$]\label{lem:multi_layer_tensor_attention} Suppose that for every $i \in [m]$, the function $g_i$ in $\mathsf{TF}$ can be evaluated by a $\poly(n)$ size constant depth $d_g$  uniform threshold circuit. Assuming that $p \leq \poly(n)$, the $\rope$-based tensor attention $\mathsf{TF}$, as defined in Definition~\ref{def:multi_layer_self_attn}, is evaluatable by $\poly(n)$ size uniform threshold circuit of and depth $(m+1) d_g + 6 m d_{\mathrm{std}} + 5 m d_{\oplus} + m d_{\mathrm{exp}}$. \end{lemma}

\begin{proof} 
By assumption, $\forall i \in [m]$, 
$g_i$ is evaluatable by a $\poly(n)$ size constant $d_g$ depth uniform threshold circuit.
From Lemma~\ref{lem:single_tensor_attention}, the attention operation $\mathsf{Attn}_i$ is evaluatable by $\poly(n)$ size uniform threshold circuit with depth $6 d_{\mathrm{std}} + 5 d_{\oplus} + d_{\mathrm{exp}}$. 

In order to compute $\mathsf{TF}(X)$, the functions $g_0, g_1, \ldots, g_m$ and $\mathsf{Attn}_1, \ldots, \mathsf{Attn}_m$ must be evaluated. Consequently, the overall depth is $(m+1) d_g + 6 m d_{\mathrm{std}} + 5 m d_{\oplus} + m d_{\mathrm{exp}}$, and the circuit size remains $\poly(n)$.

The proof is completed.
\end{proof}

\subsection{Circuit Complexity Bound of Tensor Attention}\label{sec:circuit_complexity_bound_tensor_attention}

The subsequent discussion focuses on presenting the main result regarding the circuit complexity bound for tensor attention Transformers.

\begin{theorem}[Circuit Complexity of Tensor Attention]\label{thm:main_result_tensor_attention_tc0}
Assume that for every $i \in [m]$, the function $g_i$ in $\mathsf{TF}$ is evaluatable by $\poly(n)$ size uniform threshold circuit of constant $d_g$ depth. As described in Definition~\ref{def:multi_layer_self_attn}, we can approximate the $\rope$-based tensor attention Transformer $\mathsf{TF}$ by a uniform $\TC^0$ circuit family, when $d \leq O(n)$, $p \leq \poly(n)$, and $m \leq O(1)$.
\end{theorem}

\begin{proof}
With constant $m$, and Lemma~\ref{lem:multi_layer_tensor_attention}, the depth of the circuit computing $\mathsf{TF}(X)$ is 
\begin{align*}
(m+1) d_g + 6 m d_{\mathrm{std}} + 5 m d_{\oplus} + m d_{\mathrm{exp}} = O(1),
\end{align*}
and the $\poly(n)$ circuit size. Thus, a uniform $\TC^0$ circuit family can simulate this computation.

This concludes the proof.
\end{proof}

Above Theorem~\ref{thm:main_result_tensor_attention_tc0}, we establish that, unless $\TC^0 = \NC^1$, a constant depth tensor attention with $\poly(n)$ size, and $\poly(n)$-precision can be approximated by a $\dlogtime$-uniform $\TC^0$ circuit family. While tensor attention Transformers exhibit strong empirical performance, this result indicates inherent limits in their expressivity when viewed through the framework of circuit complexity. These constraints are examined further in Section~\ref{sec:hardness}, in tandem with the analysis from Section~\ref{sec:complexity_each_step_rope}. 

\section{Complexity of \texorpdfstring{$\rope$}{}-based Tensor Attention Transformer}\label{sec:complexity_each_step_rope}

This section presents key results concerning the circuit complexity of fundamental operations within $\rope$-based tensor attention computations. Section~\ref{sec:approx:trig} investigates trigonometric functions, which play a crucial role in rotary position embeddings, while Section~\ref{sec:compute_rope_tensor_attention_matrix} focuses on the $\rope$-based tensor attention matrix computation. Section~\ref{sec:compute_single_rope_tensor_attention_layer} delves into the individual $\rope$-based tensor attention layer, whereas Section~\ref{sec:other_conponents} explores other components beyond the attention layer. In Section~\ref{sec:compute_tensor_attention_transformer}, the complete $\rope$-based tensor attention mechanism is detailed. Finally, Section~\ref{sec:circuit_complexity_bound} presents the primary results regarding the circuit complexity bounds for $\rope$-based tensor attention, forming the foundation for the essential theorem on {$\rope$}{}-based Tensor Attention Transformer expressiveness.

\subsection{Approximating Trigonometric Functions}\label{sec:approx:trig}

Here, we outline the efficient calculation of fundamental trigonometric functions that are critical for $\rope$ embeddings via threshold circuits. The next lemma plays a central role:

\begin{lemma}[Trigonometric Function Approximation in $\TC^0$, Lemma 4.1 in~\cite{cll+24_rope}]\label{lem:sin_cos} For any $p \leq \poly(n)$, the values of $\sin(x)$ and $\cos(x)$ for a float point number $x$ of $p$ bits with a relative error bounded by $2^{-p}$ are evaluatable by $\poly(n)$ size uniform threshold circuit with constant depth. Let $d_{\triangle}$ denote the maximum depth required to calculate both $\cos(x)$ and $\sin(x)$.
\end{lemma}

\subsection{\texorpdfstring{$\rope$}{}-based Tensor Attention Matrix}\label{sec:compute_rope_tensor_attention_matrix}

The following section builds on what we already know about the computation of the $\rope$-based tensor attention matrix.

\begin{lemma}[$\rope$-based tensor attention matrix computation in $\TC^0$]\label{lem:rope_attention_matrix} For any polynomial $p \leq \poly(n)$, a size $\poly(n)$ uniform threshold circuit with depth $7d_{\mathrm{std}} + 4d_\oplus + d_\triangle + d_{\exp}$ is capable of computing  $A$, i.e., the attention matrix in Definition~\ref{def:rope_attn_matrix}. \end{lemma}

\begin{proof} For every $j_1, j_2, j_3 \in [n]$, the matrix element $A_{j_1,j_2 + (j_3 - 1)d}$ is evaluated according to the formula in Definition~\ref{def:rope_attn_matrix}.

From Lemma~\ref{lem:matrix_multi}, the matrix products $Q := Z W_Q$, $K_1 := ZW_{K_1}$, and $K_2 := ZW_{K_2}$ can be evaluated in parallel by a  size $\poly n$ depth $d_{\mathrm{std}} + d_\oplus$ uniform threshold circuit.

As indicated by Lemma~\ref{lem:sin_cos}, the entries of $R_{j_1-j_2}$ are evaluatable by a size $\poly(n)$ depth $d_\triangle$ uniform threshold circuit. Since $n$ is polynomial, all entries of $R_{j_1 - j_2}$ are evaluatable simultaneously with the same circuit size and depth. This holds true for $R_{j_1 - j_3}$ and $R_{j_1 - j_2}$ as well.

According to Lemma~\ref{lem:row_kronecker_product}, the row-wise Kronecker product $R_{j_1,j_2 + (j_3 - 1)d} = R_{j_1 - j_2} \ominus R_{j_1 - j_3}$ is evaluatable by $\poly n$ size uniform threshold circuit with $d_{\mathrm{std}}$ depth.

Lemma~\ref{lem:kronecker_product} further shows that the Kronecker product $K := K_1 \otimes K_2$ can be evaluated using a size $\poly n$ depth $d_{\mathrm{std}}$ uniform threshold circuit.

By Lemma~\ref{lem:matrix_multi} and the first part of Lemma~\ref{cor:floor_float_TC}, the matrix product and division $Q R_{j_1,j_2 + (j_3 - 1)d} K^\top / d$ is evaluatable by $\poly (n)$ size uniform threshold circuit with $3 d_{\mathrm{std}} + 2 d_\oplus$ depth.

The exponential function $\exp()$ can be evaluated using Lemma~\ref{lem:exp} by a size $\poly n$ depth $d_{\exp}$ uniform threshold circuit.

Thus, the total required depth to compute the matrix $A$ is:

$$7d_{\mathrm{std}} + 4d_\oplus + d_\triangle + d_{\exp}.$$

Any entry of $A_{i,j}$,  $\forall i, j \in [n]$ can be evaluated in parallel, so the overall circuit size is $\poly(n)$, and the total depth is $7d_{\mathrm{std}} + 4d_\oplus + d_\triangle + d_{\exp}$. 

The proof is thus concluded.
\end{proof}

\subsection{Single \texorpdfstring{$\rope$}{}-based Tensor Attention Layer}\label{sec:compute_single_rope_tensor_attention_layer}
This section provides a detailed examination of the $\rope$-based tensor attention layer, with an emphasis on tracking the circuit depth requirements throughout the computation process.

\begin{lemma}[Single $\rope$-based Attention Layer within $\TC^0$]\label{lem:attn} For $p \leq \poly(n)$, the $\mathsf{Attn}$ defined in Definition~\ref{def:single_rope_tensor_attention}, can is evaluatable by a size $\poly(n)$ depth $11d_{\mathrm{std}} + 8d_\oplus + d_\triangle + d_{\exp}$ uniform threshold circuit. 
\end{lemma}

\begin{proof} To evaluate $\mathsf{Attn}$, the multiplication of the matrices $D^{-1}$, $A$, and $V$ is required. Initially, $D := \diag(A {\bf 1}n)$ can be evaluated by $\poly(n)$ size uniform threshold circuit with $d_\oplus$ depth, as established in Part 3 of Lemma~\ref{lem:float_operations_TC}. The matrix $A$ requires a circuit with $7d_{\mathrm{std}} + 4d_\oplus + d_\triangle + d_{\exp}$ depth, according to Lemma~\ref{lem:rope_attention_matrix}.

Next, the evaluation of $V := V_1 \oslash V_2$ is carried out in depth $d_{\mathrm{std}}$, as per Lemma~\ref{lem:column_kronecker_product}. The multiplication of $A$ and $V$ is performed by $\poly(n)$ size uniform threshold circuit of $d_{\mathrm{std}} + d_\oplus$ depth, based on Lemma~\ref{lem:matrix_multi}.

Lastly, the multiplication $D^{-1} \cdot A V$ is evaluated by performing division in parallel, which is implemented by $\poly(n)$ size uniform threshold circuit of $d_{\mathrm{std}} + d_\oplus$ depth, as per Part 1 of Lemma~\ref{lem:float_operations_TC}. 

Summing the circuit depths gives:

$$11d_{\mathrm{std}} + 8d_\oplus + d_\triangle + d_{\exp}.$$

Because parallel operations can be conducted for each element, the attention operation $\mathsf{Attn}(X)$ can be evaluated by a uniform threshold circuit with the required depth and size.

This concludes the proof.
\end{proof}

\subsection{Building Blocks other than \texorpdfstring{$\rope$}{}-based Tensor Attention Layer}\label{sec:other_conponents}

According to Definition~\ref{def:multi_layer_self_attn}, the definition of the Multi-layer $\rope$-based Transformer is provided, which integrates $\rope$-based self-attention layers together with supplementary components, such as layer normalization and MLP. This section subsequently addresses the circuit complexity associated with these mechanisms.

The analysis begins with an investigation of the complexity pertaining to the MLP layer.

\begin{lemma}[Compute MLP in $\mathsf{TC}^0$] If $p \leq \poly(n)$, $\poly(n)$ size depth $2d_\mathrm{std} + d_{\oplus}$ uniform threshold circuit suffices to evaluate the MLP layer as defined in Definition~\ref{def:mlp}. \end{lemma}

\begin{proof} For each $i \in [m]$, evaluating $W X_{i,}$ necessitates a size $\poly(n)$ circuit with a depth of $d_\mathrm{std} + d_{\oplus}$, as demonstrated in Lemma~\ref{lem:matrix_multi}. To compute $W X_{i,} + b$, an extra depth of $d_\mathrm{std}$ and size $\poly(n)$ is required. Thus, the cumulative depth becomes $2d_\mathrm{std} + d_{\oplus}$, while the size remains $\poly(n)$. Since the computation can be parallelized overall $i \in [m]$, the proof is complete. \end{proof}

Next, we will turn our attention to the complexity of the LN layer.

\begin{lemma}[Compute Layer-norm in $\mathsf{TC}^0$] Let $p \leq \poly(n)$. Then, the layer-normalization defined in Definition~\ref{def:layer_norm} can be evaluated by $\poly(n)$ size depth $5d_\mathrm{std} + 2d_{\oplus} + d_\mathrm{sqrt}$ uniform threshold circuit. 
\end{lemma}

\begin{proof} For each $i \in [n]$, the computation of the mean $\mu_i$ requires a circuit with depth $d_\mathrm{std} + d_{\oplus}$ and size $\poly(n)$, as shown in Lemma~\ref{lem:float_operations_TC}. An additional depth of $3d_\mathrm{std} + d_{\oplus}$ and size $\poly(n)$ is required to compute the variance $\sigma_i^2$. Finally, the computation of $g^\mathrm{LN}(x)_{i,*}$ requires a depth of $2d_\mathrm{std} + d_\mathrm{sqrt}$ and size $\poly(n)$, based on Lemmas~\ref{lem:float_operations_TC} and~\ref{lem:sqrt}.

Adding these contributions together, the total depth is $6d_\mathrm{std} + 2d_{\oplus} + d_\mathrm{sqrt}$, while the size remains $\poly(n)$. As this computation can be parallelized across all $i \in [n]$, we complete this proof. \end{proof}

\subsection{Multi-layer \texorpdfstring{$\rope$}{}-based Tensor Attention}\label{sec:compute_tensor_attention_transformer}

We now describe the computation of the multi-layer $\rope$-based tensor attention Transformer.

\begin{lemma}[Multi-layer $\rope$-based tensor attention Transformer computation in $\TC^0$]\label{lem:tf}
Consider the assumption that for every $i \in [m]$, $g_i$ in $\mathsf{TF}$ can be evaluated using $\poly(n)$ size uniform threshold circuit with a constant depth $d_g$. When $p \leq \poly(n)$, the $\rope$-based tensor attention $\mathsf{TF}$, as specified in Definition~\ref{def:multi_layer_self_attn}, can be evaluated by $\poly(n)$ size uniform threshold circuit of depth $(m+1)d_g + 11 m d_\mathrm{std} + 8 m d_\oplus + m (d_\triangle + d_\mathrm{exp})$.
\end{lemma}

\begin{proof}
Under the given assumption, for every $i \in [m]$, $g_i$ can be evaluated by $\poly(n)$ size uniform threshold circuit having constant $d_g$ depth.

Moreover, from Lemma~\ref{lem:attn}, it follows that each $\mathsf{Attn}_i$ is evaluatable by $\poly(n)$ size uniform threshold circuit with depth $8d_\mathrm{std} + 6d_\oplus + d_\triangle + d_\mathrm{exp} + 1$.

To approximate $\mathsf{TF}(X)$, it is required to evaluate $g_0, g_1, \ldots, g_m$ and $\mathsf{Attn}_1, \ldots, \mathsf{Attn}_m$. As a result, the total depth of the $\poly(n)$ size circuit is $(m+1)d_g + 11 m d_\mathrm{std} + 8 m d_\oplus + m (d_\triangle + d_\mathrm{exp})$.

This concludes the proof.
\end{proof}

\subsection{Circuit Complexity of \texorpdfstring{$\rope$}{}-based Tensor Attention}\label{sec:circuit_complexity_bound}

Here, we present the central contribution of this paper, which establishes the circuit complexity for the $\rope$-based tensor attention. 

\begin{theorem}[Main result, Circuit complexity of $\rope$-based tensor attention Transformers]\label{thm:main_result_tc0}
Assume that $ \forall i \in [m]$, $g_i$ in $\mathsf{TF}$ can be computed using $\poly(n)$ size uniform threshold circuit of constant depth $d_g$. The $\rope$-based tensor attention $\mathsf{TF}$, as defined in Definition~\ref{def:multi_layer_self_attn}, is simulatable by uniform $\TC^0$ circuit family when $d \leq O(n), p \leq \poly(n),$ and $ m \leq O(1)$. 
\end{theorem}

\begin{proof}
According to Lemma~\ref{lem:tf}, we have $m = O(1)$, the $O(\poly(n))$ bounded circuit used to compute $\mathsf{TF}(X)$ has a depth given by

\begin{align*} (m+1)d_g + 11 m d_\mathrm{std} + 8 m d_\oplus + m (d_\triangle + d_\mathrm{exp}), \end{align*}
which bounded by $O(\poly(n))$. 
Thus, based on the definition of $\mathsf{TC}^0$, it follows that the uniform $\TC^0$ circuit family can approximate $\rope$-based tensor attention Transformer. 

The proof is complete.
\end{proof}

In Theorem~\ref{thm:main_result_tensor_attention_tc0} and Theorem~\ref{thm:main_result_tc0}, unless $\TC^0 = \NC^1$, a $\dlogtime$-uniform $\TC^0$ circuit family can emulate both tensor attention Transformers and $\rope$-based tensor attention Transformers, which are defined by constant depth, $\poly(n)$ precision, and $\poly(n)$ size. This finding suggests that, notwithstanding the empirical success of these models, their expressive capabilities are intrinsically constrained when analyzed through the lens of circuit complexity. The subsequent section will delve deeper into these limitations.

%% file: 5_hardness.tex
\section{Hardness}\label{sec:hardness}

This section delineates two fundamental problems, accompanied by their respective hardness results. The fixed membership problem is introduced in Section~\ref{sec:fixed_menbership}, while the closure problem is defined in Section~\ref{sec:a_closure}. Section~\ref{sec:hardness_result} presents the four principal hardness results.

\subsection{Fixed Membership}\label{sec:fixed_menbership}
The fixed membership problem, as originally formulated in~\cite{fk19}, is thoroughly defined in this section. A formal exposition of its definition is provided as the foundation for subsequent analysis.

\begin{definition}[Fixed membership problem, Definition from~\cite{fk19}]\label{def:fixed_menbership}
The fixed membership problem is defined as follows:
\begin{itemize}
\item Input: A fixed morphism $h$: $A^+ \to S$, a fixed set $P \subseteq F(S)$ and finite words $u, v \in A^+$ 
\item Question: Is $uv^\omega \in [P]$? 
\end{itemize}
$F(S)$ denotes the collection of finite subsets of $S$. 
\end{definition}

\begin{proposition}[Proposition 7.1 from~\cite{fk19}]\label{pro:fixed_menbership_nc1complete}
The fixed membership problem for recognizing morphisms over finite words is $\NC^1$-complete.

\end{proposition}

\subsection{\texorpdfstring{$(A_{F,r})^*$}~ Closure}\label{sec:a_closure}
In this section, attention is shifted to the $(A_{F,r})^*$ closure problem, as introduced in~\cite{aam03}.

\begin{definition}[Kleene star, page 3 of~\cite{kuz21}, Definition 7.1 from~\cite{aam03}]\label{def:kleene_star}
Let $L$ be a language, the kleene star of $L$, denoted by $L^*$, is the set of all finite concatenations of strings from $L$, defined as: 
\begin{align*}
L^* = \sup_{\preceq} \{L^n \mid n \geq 0\}
\end{align*}
where $L^0 := {\epsilon}$. 
\end{definition}

\begin{definition}[$(A_{F,r})^*$ Closure Problem, Definition 7.1 from~\cite{aam03}]\label{def:a_closure}
Let $(A, \circ)$ denote a finite monoid.\footnote{For convenience, $(A, \circ)$ is abbreviated as $A$.} A natural homomorphism $v: A^* \rightarrow A$ maps each word $w$ to its corresponding valuation $v(w)$ in the monoid $A$. Let $F \subseteq A$ and $r$ be a positive integer. The language $A_{F,r} \subseteq A^*$ is characterized by $A_{F,r} = \{ w \in A^* \mid \|w \| \leq r, v(w) \in F \}$. The $(A_{F,r})^*$ closure problem refers to the decision problem aimed at determining whether a given string $s$ belongs to $(A_{F,r})^*$.
\end{definition}

\begin{theorem}[Theorem 7.3(a) from~\cite{aam03}]\label{the:a_closure_nc1complete}
Assume that $A$ is a nonsolvable monoid. Then, there exists a group $F \subseteq A$ and a constant $r > 0$ such that the $(A_{F, r})^*$ closure problem is $\NC ^1$-complete.
\end{theorem}

\subsection{Hardness Result}\label{sec:hardness_result}
This part presents four crucial findings concerning tensor attention Transformers and $\rope$-based tensor attention Transformers.

\begin{theorem}\label{thm:tc_fixed_membership}
If $\TC^0 \neq \NC^1$, $O(1)$ layers $\rope$-based tensor attention Transformer with $d \leq O(n)$ hidden dimension, $\poly(n)$ precision is incapable of solving the fixed membership problem. 
\end{theorem}

\begin{proof}
The proof follows from the combination of Theorem~\ref{thm:main_result_tc0}, which provides a circuit complexity bound for $\rope$-based tensor attention Transformers, and Proposition~\ref{pro:fixed_menbership_nc1complete}, which establishes that the fixed membership problem for recognizing morphisms over finite words is $\NC^1$-complete. Additionally, Fact~\ref{fact:folklore}, which outlines the hierarchy of circuit families, is also applied here. This completes the proof.
\end{proof}

\begin{theorem}\label{thm:tc_fixed_membership_tensor_attention}
Unless $\TC^0 = \NC^1$, it is not possible for a $O(1)$ layers tensor attention Transformer with $d \leq O(n)$ hidden dimension and $\poly(n)$ precision to address the fixed membership problem.
\end{theorem}

\begin{proof}
The result is derived by combining Theorem~\ref{thm:main_result_tensor_attention_tc0} (which provides the circuit complexity bound for tensor attention Transformers), Proposition~\ref{pro:fixed_menbership_nc1complete} (demonstrating the $\NC^1$-completeness of the fixed membership problem), and Fact~\ref{fact:folklore} (pertaining to the structure of circuit families). As a result, this proof is complete.
\end{proof}

\begin{theorem}\label{thm:tc_a_closure}
Assuming $\TC^0 \neq \NC^1$, a $O(1)$ layers tensor attention Transformer with $d \leq O(n)$ hidden dimension, and $\poly(n)$ precision is not capable of solving the $(A_{F,r})^*$ closure problem. 
\end{theorem}

\begin{proof}
This follows directly from Theorem~\ref{thm:main_result_tc0}, which establishes the circuit complexity bound for $\rope$-based tensor attention Transformers, and Theorem~\ref{the:a_closure_nc1complete}, which asserts that the $(A_{F,r})^*$ closure problem is $\NC^1$-complete. Additionally, Fact~\ref{fact:folklore} concerning the hierarchy of circuit families is also utilized. Thus, the proof is complete.
\end{proof}

\begin{theorem}\label{thm:tc_a_closure_tensor_attention}
Unless $\TC^0 = \NC^1$, it is not possible for a tensor attention Transformer with $O(1)$ layers, $\poly(n)$-precision, and $d \leq O(n)$ hidden dimension to solve the $(A_{F,r})^*$ closure problem.
\end{theorem}

\begin{proof}
Theorem~\ref{thm:main_result_tensor_attention_tc0} provides this result upon application, which provides the circuit complexity bound for tensor attention Transformers, Theorem~\ref{the:a_closure_nc1complete}, which proves the $\NC^1$-completeness of the $(A_{F,r})^*$ closure problem, and Fact~\ref{fact:folklore}, which discusses the hierarchy of circuit families. Therefore, the proof is concluded.
\end{proof}

%% file: 6_conclusion.tex
\section{Conclusion}\label{sec:conclusion}

This paper presents a comprehensive theoretical investigation into tensor attention Transformers and their $\rope$-based variants, specifically addressing the computational implications of their attention mechanisms. A systematic analysis was performed, examining the circuit complexity of the various components, including trigonometric functions, tensor operations, and the full attention mechanisms. It was established that these models are simulable by uniform $\mathsf{TC}^0$ circuits. Moreover, under the assumption that $\mathsf{TC}^0 \neq \mathsf{NC}^1$, it was shown that $O(1)$ layers tensor attention Transformers and $\rope$-based tensor attention Transformers with $\mathrm{poly(n)}$ precision, and $d \leq O(n)$ hidden dimensions are unsuccessful in solving fixed membership problems or $(A_{F,r})^*$ closure problems.
This conclusion is of particular significance, as it exposes the fundamental limits in expressivity inherent to tensor attention and $\rope$-based tensor attention mechanism. Notwithstanding their broad empirical success in many state-of-the-art models, these mechanisms are fundamentally constrained in terms of their computational capabilities. The results emphasize the trade-off between computational efficiency and expressive power, suggesting that certain complex tasks remain out of reach for these models under the current assumptions.

It is important to note that our analysis is primarily confined to forward computations and assumes constant-depth nonlinear activation functions. This leaves the impact of training dynamics, alternative activation functions, and different formulations of tensor attention unexplored. Future work could expand on this analysis to investigate alternative positional encoding schemes, more advanced attention models, or different activation functions, potentially uncovering whether these complexity boundaries hold for other Transformer variants.
Ultimately, the findings presented here reveal a fascinating discrepancy between tensor attention models' theoretical limitations and their empirical performance. Understanding how these models achieve practical effectiveness despite their theoretical shortcomings could stimulate the development of new, theoretically robust design principles. Such insights are essential for advancing neural network architectures that maintain a balance between rigorous theoretical foundations and empirical success, fostering the creation of more scalable and powerful models.